\documentclass[letterpaper, 10 pt, conference]{ieeeconf}
\IEEEoverridecommandlockouts                              
\usepackage{color}
\usepackage[utf8]{inputenc}
\usepackage[english]{babel}

\usepackage{amsmath} 
\usepackage{amsthm}
\usepackage{amssymb}  
\usepackage{multicol}
\usepackage{lipsum}
\usepackage{xcolor}
\usepackage{graphicx}
\usepackage{placeins}
\usepackage{subcaption}
\usepackage{algorithm}
\usepackage[noend]{algorithmic}
\newcommand{\fig}[1]{Fig.~\ref{#1}}
\newcommand{\tab}[1]{Table~\ref{#1}}
\newcommand{\eq}[1]{(\ref{#1})}
\newtheorem{thm}{Theorem}[section]
\newtheorem{lem}[thm]{Lemma}

\allowdisplaybreaks
\title{\LARGE \bf
Chance-Constrained Optimization in Contact-Rich Systems\\ for Robust Manipulation 
}

\author{Yuki Shirai$^{\dagger}$, Devesh K. Jha$^{\ddagger}$, Arvind U. Raghunathan$^{\ddagger}$, Diego Romeres$^{\ddagger}$%
\thanks{$^{\dagger}$ Yuki Shirai is with the Department of Mechanical and Aerospace Engineering, University of California, Los Angeles, CA, USA 90095 {\tt\small yukishirai4869@g.ucla.edu}}%
\thanks{$^{\ddagger}$Devesh K. Jha, Arvind U. Raghunathan and Diego Romeres are with Mitsubishi Electric Research Laboratories (MERL), Cambridge, MA, USA 02139 {\tt\small \{jha,raghunathan,romeres\}@merl.com}}}%

\begin{document}

\maketitle
\thispagestyle{empty}
\pagestyle{empty}
\begin{abstract}
 This paper presents a chance-constrained formulation for robust trajectory optimization during manipulation. In particular, we present a chance-constrained optimization for Stochastic Discrete-time Linear Complementarity Systems (SDLCS). To solve the optimization problem, we formulate Mixed-Integer Quadratic Programming with Chance Constraints (MIQPCC). In our formulation, we explicitly consider joint chance constraints for complementarity as well as states to capture the stochastic evolution of dynamics.
We evaluate robustness of our optimized trajectories in simulation on several systems. The proposed approach outperforms some recent approaches for robust trajectory optimization for SDLCS. 
\end{abstract}
\section{Introduction}\label{sec:introduction}


Contacts are central to manipulation problems. Consequently, contact modeling has been an active area of research in robotics since the last several decades 
\cite{todorov2010implicit, drumwright2011evaluation, drumwright2010modeling, anitescu1997formulating, 9113247}.
One of the most popular approaches to model contact dynamics is using Linear Complementarity Problem (LCP). LCP models are widely used for modeling contact dynamics in academia as well as in several physics simulation engines such as Bullet, ODE, etc. Trajectory optimization (TO) of LCP-based contact models has been used for manipulation~\cite{DBLP:journals/corr/abs-2106-03220, jin2021trajectory} and legged locomotion~\cite{posa2014direct}. Lyapunov stability of linear systems with complementarity systems has also been studied~\cite{CamlibelPangShen,DBLP:journals/corr/abs-2008-02104,RaghunathanLinderoth}. However, most of these works assume deterministic contact models to perform TO. In reality, frictional interaction systems suffer from several uncertainties which lead to stochastic dynamics and thus, it is important to consider uncertainty during TO. Modeling uncertainty in LCP-based contact models leads to Stochastic Discrete-time Linear Complementarity System (SDLCS). 
\begin{figure}
    \centering
    \includegraphics[scale=0.45]{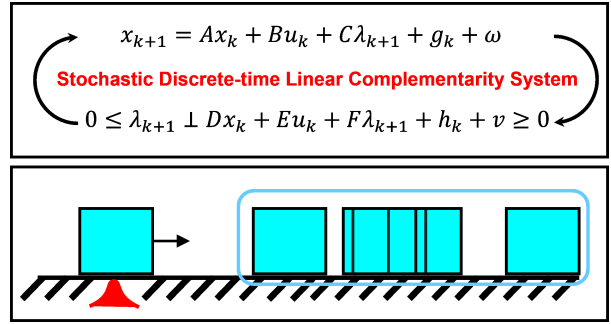}
    \caption{This paper presents chance-constrained optimization for SDLCS. The figure shows the case of a sliding box on a plane where the coefficient of friction is a Gaussian random variable.  Note that $w$ and $v$ are additive uncertainty terms.}
    \label{fig:stochastic_compl_constraint}
\end{figure}

\fig{fig:stochastic_compl_constraint} pictorially shows a SDLCS we study in this paper. We consider the SDLCS that has uncertainty in parameters and additive noises in dynamics and complementarity constraints.  As shown in \fig{fig:stochastic_compl_constraint}, one should notice that uncertainty leads to stochastic evolution of system states in SDLCS. Thus, a robust optimization formulation should consider the uncertainty in state evolution. In some recent works that consider stochastic complementarity constraints, an expected residual minimization (ERM)-based~\cite{chen2005expected} penalty is used to solve the robust optimization problem~\cite{drnach2021robust}. A major shortcoming of such an approach is that it fails to capture the stochastic evolution of system dynamics due to the stochastic complementarity constraint. 
In~\cite{DBLP:journals/corr/abs-2105-09973}, the authors augment the formulation in \cite{drnach2021robust} with chance constraints. However, this formulation has certain fundamental shortcomings which prevent constraint satisfaction guarantees. We present a formulation that circumvents these shortcomings by using a mixed integer formulation. Using some relaxation of the original joint chance complementarity constraint problem for the SDLCS, the resulting problem can be solved using mixed integer programming.


 In this paper, we present a formulation of robust trajectory optimization for SDLCS. 
 \textcolor{black}{Since worst-case robust optimization is quite conservative and does not explicitly discuss stochastic evolution of states \cite{ben2009robust}, this work considers probabilistic optimization with stochastic evolution of states.}
 Robustness to uncertainty is provided by enforcing probabilistic satisfaction of state constraints. Under certain simplifications, we show that the chance-constrained problem can be reformulated as a Mixed Integer Quadratic Program with Chance Constraints. 

\textbf{Contributions.} This paper has the following contributions:
\begin{enumerate}
    \item We present a novel formulation for chance-constrained optimization of SDLCS. 
    \item We compare our proposed approach with several previously proposed techniques and demonstrate that our method outperforms the recent techniques in \cite{drnach2021robust, DBLP:journals/corr/abs-2105-09973}.
\end{enumerate}
The proposed algorithm is demonstrated on several manipulation systems with linear dynamics. 

\section{Related Work}\label{sec:related_work}
In this section, we review some of the work which is most close to the work presented in this paper. Our work is closely related to TO techniques for contact-rich systems. Contact-implicit TO techniques are becoming very popular for performing TO for contact-rich systems, and several techniques have been proposed for manipulation as well as legged locomotion~\cite{DBLP:journals/corr/abs-2106-03220,posa2014direct, manchester2020variational}. All the above techniques assume perfect model knowledge and do not consider uncertainty. 

There has been some recent work on robust TO in SNCS~\cite{drnach2021robust, DBLP:journals/corr/abs-2105-09973, yuki2022pivot}.  In~\cite{drnach2021robust}, the authors have utilized the formulation of ERM for robust TO. ERM, first introduced in~\cite{chen2005expected} for Stochastic Linear Complementarity Problem (SLCP), aims at minimizing the expected error in satisfying the SLCP. 
In~\cite{drnach2021robust}, authors use ERM as an additional penalty term in their TO problem. However, such a formulation does not consider the stochastic state evolution of the system during optimization. 
A chance-constrained formulation for SNCS is presented in~\cite{DBLP:journals/corr/abs-2105-09973}. This method augments the ERM-augmented objective in~\cite{drnach2021robust} with additional chance constraints on satisfying the complementarity constraints. 
The formulation ignores the stochastic evolution of system state during optimization, and thus borrows the limitations of~\cite{drnach2021robust}. Furthermore, this formulation is incapable of enforcing a constraint violation probability smaller than $0.5$ for any degree of uncertainty. Consequently, this method is very fragile for trajectories with the horizon, $N>1$ as the chance of violating the constraints for such trajectories is $0.5N \geq 1$ \cite{doi:10.1287/opre.36.1.145}. Our formulation addresses these weaknesses under certain simplifying assumptions for SDLCS. 

Another line of work which is relevant to understand some of our proposed work is related to chance-constrained optimization (CCO). This has been extensively studied in robotics as well as optimization literature\cite{5970128, 8651541, 9145612}. In \cite{5970128}, authors have proposed stochastic optimization formulation for open-loop collision avoidance problems using chance constraints under Gaussian noise.  \cite{9145612} uses statistical moments of the distribution to handle non-Gaussian chance constraints. An important point to note here is that in all CC formulation for dynamic optimization, one needs to consider the CDF function for the joint probability distribution of all variables. However, such distribution is generally extremely challenging to compute.
Thus, in general, the joint chance constraint is decomposed into individual chance constraints using  Boole's inequality (see~\cite{5970128,9145612}), which results in very conservative approximation of the individual constraints. Our formulation utilizes Boole's inequality to convert the original computationally intractable joint chance constraints into conservative but tractable independent chance constraints.  


\section{Problem Preliminary}\label{sec:prob_formulation}
For the completeness of the paper, we provide a brief introduction to linear complementarity problem and its stochastic form. This is followed by a problem formulation for robust trajectory optimization for linear dynamical systems with stochastic complementarity solution. We also point several key differences of our approach from previous attempts for robust trajectory optimization for stochastic complementarity system.
\subsection{Discrete-time Linear Complementarity System (DLCS)}\label{sec:define_lcp}

A DLCS is a discrete-time linear dynamical system with complementarity constraints~\cite{RaghunathanLinderoth} given by:
\begin{subequations}
\begin{flalign}
x_{k+1}=A x_k+B u_k+C \lambda_{k+1}+g_k \label{dlcs.dyn} \\
0 \leq \lambda_{k+1} \perp D x_k+E u_k+F \lambda_{k+1}+h_k \geq 0 \label{dlcs.lcp}
\end{flalign}\label{dlcs}
\end{subequations}
where $k$ is the time-step index, $x_k\in \mathbb{R}^{n_{x}}$ is the state, $u_k\in \mathbb{R}^{n_{u}}$ is the control input, and $\lambda_k \in \mathbb{R}^{n_{c}}$ is the algebraic variable (e.g., contact forces). In addition, $A \in \mathbb{R}^{n_{x} \times n_{x}}$, $B \in \mathbb{R}^{n_{x} \times n_{u}}$, $C \in \mathbb{R}^{n_{x} \times n_{c}}$, $g_k \in \mathbb{R}^{n_{x}}$, $D \in \mathbb{R}^{n_{c} \times n_{x}}$, $E \in \mathbb{R}^{n_{c} \times n_{u}}$, $F \in \mathbb{R}^{n_{c} \times n_{c}}$, and $h_k \in \mathbb{R}^{n_{c}}$. The $i$-th element of vector $p_k$ ($p_k$ can be $x_k, u_k, \lambda_k$) is represented as $p_{k,i}$.
The $i$-th diagonal element of matrix $P_k$ is represented as $P_{k,ii}$.
The notation $0 \leq a \perp b \geq 0$ denotes the complementarity constraints $a \geq 0, b \geq 0, a b=0$.

Given a $x_k,u_k$, an unique solution $\lambda_{k+1}$ to~\eqref{dlcs.lcp} exists if $F$ is P-matrix \cite{lcpbook}. If $F$ does not satisfy the P-matrix property, it is possible that $\lambda_{k+1}$ satisfying~\eqref{dlcs.lcp} is non-unique or non-existent. 

\subsection{Contact-Implicit Trajectory Optimization}
A contact-implicit trajectory optimization for the DLCS can be formulated as:
\begin{subequations}
\begin{align}
\min _{x, u, \lambda} &\; \sum_{k=0}^{N-1} J(x_k, u_k, \lambda_k)\\
\text{s. t. } &\; x_{k+1}=A x_k+B u_k+C \lambda_{k+1}+g_k, \\
&\; 0 \leq \lambda_{k+1} \perp D x_k+E u_k+F \lambda_{k+1}+h_k \geq 0,\\
&\; \quad x_{0} = x_s, x_N = x_g, 
x_{k} \in \mathcal{X}, u_{k} \in \mathcal{U}, \lambda_k \leq \lambda_{u}
\end{align}
\label{equation_lcp}
\end{subequations}
where $x_s, x_g$ represent the initial and the terminal values, respectively, $\mathcal{X} \subseteq \mathbb{R}^{n_{x}}$ and $\mathcal{U} \subseteq \mathbb{R}^{n_{u}}$ are convex polytopes consisting of a finite number of linear inequality constraints, $\lambda_{u}$ is the upper bound of $\lambda_k$, and $N$ is the time horizon.

While \eq{equation_lcp} is widely used in various robotic applications, it can be fragile under uncertainty, which is often the case in model-based manipulation. Hence, we consider a novel formulation of \eq{equation_lcp} so that the generated trajectory from the optimization would be robust under uncertainty.

\subsection{Stochastic Discrete-time Linear Complementarity Systems (SDLCS)}\label{SDLCS_sec}
We consider the following SDLCS, i.e. DLCS with uncertainty:
\begin{subequations}
\begin{align}
x_{k+1}=A x_k+B u_k+C \lambda_{k+1}+g_k + w_k \label{slcp1} \\
0 \leq \lambda_{k+1} \perp y_{k+1} \geq 0 \label{slcp2}
\end{align}
\label{SDLCS_equations}
\end{subequations}
where $y_{k+1} = D x_k+E u_k+F \lambda_{k+1}+h_k + v_k$. $w_{k} \in \mathbb{R}^{n_{x}}, v_{k} \in \mathbb{R}^{n_{c}}$ are known additive uncertainty. 
We consider the case where the coefficient matrix $C$ in \eq{slcp1} and $F$ in \eq{slcp2} are  stochastic matrices to discuss a more realistic stochastic effect due to complementarity constraints. This corresponds to the case when one might have uncertainty arising from parameter identification leading to a SDLCS. An alternative to this is to allow the complementarity variable $\lambda_{k+1}$ to be stochastic. However, such treatment is out of the scope of the current work. 
%
 Our treatment of SDLCS leads to stochastic evolution of system states $x_k$, while we treat $\lambda_{k+1}$ as deterministic. The assumption of determinacy in $\lambda_{k+1}$ is similar to several previous works \cite{chen2005expected}, \cite{drnach2021robust}, \cite{DBLP:journals/corr/abs-2105-09973}, \cite{Tassa-RSS-10}. 

The authors in~\cite{drnach2021robust} use ERM to solve TO of SDLCS and use the following cost function:
\begin{equation}
\sum_{k=0}^{N-1}\left(J\left(x_{k}, u_{k}, \lambda_{k+1}\right)+\beta \mathbb{E}\left[\left\|\psi\left(\lambda_{k+1}, y_{k+1}\right)\right\|^{2}\right]\right)
\label{erm_gatech}
\end{equation}
where $\psi$ is an Nonlinear Complementarity Problem (NCP) function, $\beta$ is a weighting scalar. We compare the robustness of our formulation with \eq{erm_gatech} in Sec~\ref{sec:result}.

\section{Robust Trajectory Optimization for SDLCS}\label{sec:robust_to}



In this section, we describe our formulation for robust TO of SDLCS. 
We consider:
%
\begin{subequations}
\begin{align}
\min _{x, u, \lambda} &\; \mathbb{E} \left[\sum_{k=0}^{N-1} J(x_k, u_k, \lambda_k)\right]\\
\text{s. t. } &\; x_{k+1}=A x_k+B u_k+C \lambda_{k+1}+g_k + w_k,\label{dynamics_s}\\ 
&\; \text{Pr} \left(0 \leq \lambda_{k+1} \perp y_{k+1} \geq 0, x_k \in \mathcal{X}, \forall k \right) \geq 1-\Delta,\label{condition_s}\\
&\; x_{0} \sim \mathcal{N}\left(x_{s}, \Sigma_{s}\right), u_{k} \in \mathcal{U},  \lambda_k \leq \lambda_{u}
\end{align}
\label{equation_slcp}
\end{subequations}
where $\text{Pr}$ denotes the probability of an event and $\Delta \in \left(0, 0.5\right]$ is the user-defined maximum violation probability, where the probability of violating constraints is bounded by $\Delta$. $x_s, \Sigma_s$ are the mean and covariance matrix of the state at $k=0$. $\mathcal{X}$ and $\mathcal{U}$ are convex polytopes, consisting of a finite number of linear inequality constraints.  
In Sec~\ref{ccc_section}, we describe how we convert~\eqref{condition_s} to a tractable optimization problem.

For clarity of presentation, we explain the reasoning behind our formulation shown in~\eqref{equation_slcp}.
Since the underlying SDLCS is uncertain (also see \fig{fig:stochastic_compl_constraint}), we consider a chance-constrained formulation for optimization \textcolor{black}{ to capture stochastic evolution of states (see discussion in Sec~\ref{sec:introduction})} where we impose multiple constraints simultaneously. This is represented as joint chance constraints for the complementarity constraints as well as the states, which is succinctly written in Equation~\eqref{condition_s}). Note that we represent the chance constraints on all the variables jointly (as is common in stochastic optimization for dynamic systems) using the cdf for the state as well as complementarity variables. We show in the rest of this section how the joint constraints can be decomposed into individual chance constraints using Boole's inequality.
It is also important to note that unlike~\eqref{equation_slcp}, the method in~\cite{drnach2021robust,DBLP:journals/corr/abs-2105-09973} fails to capture the stochastic evolution of states in their formulation.

In this work, we make the following assumptions for~\eqref{equation_slcp}:
\begin{enumerate}\label{assm}
    \item Noise terms $w_k$, $v_k$ follow Gaussian distribution.
    \item The complementarity variable $\lambda_{k+1}$ is deterministic.
    \item Each element of vectors $C\lambda_{k+1}$ and $F\lambda_{k+1}$ are independent Gaussian variables.
\end{enumerate}
 We explain the rationale for above assumptions in Sec~\ref{CCC_explanation}.

\subsection{Joint Linear Chance Constraints}\label{ccc_section}
We consider \textit{joint} chance constraint 
such that multiple constraints are satisfied simultaneously with a prespecified probability.  More specifically, we consider the joint chance constraint \eq{condition_s} so that the complementarity constraints \textit{and}  state bound constraints over the whole time horizon of the optimized trajectory are satisfied with probability $1-\Delta$.
 We denote the complementarity relationship in~\eq{slcp2} succinctly as $(\lambda_{k+1,i},y_{k+1,i}) \in \mathcal{S}$ for $i = 1,\ldots,n_c$. 
Hence,  in this optimization problem \eq{equation_slcp}, we have the following joint chance constraints:
\begin{equation}
\begin{split}
&  \text{Pr} \left(0 \leq \lambda_{k+1} \perp y_{k+1} \geq 0, x_k \in \mathcal{X}, \forall k \right) \geq 1-\Delta \Longleftrightarrow\\
&\text{Pr}\left(\bigwedge_{k=0}^{N} \left(\bigwedge_{i=1}^{n_c} (\lambda_{k+1, i}, y_{k+1, i}) \in \mathcal{S} \right)
\bigwedge \left(\bigwedge_{l=1}^{L}a_l^\top x_k  \leq b_l  \right) \right) \\ &\geq 1-\Delta
\label{joint_CC}
\end{split}
\end{equation}
where $\bigwedge$ is the logical AND operator. $L$ represents the number of chance constraints involving $x$ at $k$, except for the complementarity constraints. $a_l \in \mathbb{R}^{n_{x}}$ is the constant vector and $b_l$ is a scalar.  

Obtaining a cumulative distribution function (cdf) of \eq{joint_CC} is challenging because the joint probability of states and complementarity variables is considered. 
\textcolor{black}{The only way to decompose joint chance constraints is Boole's inequality \cite{doi:10.1287/opre.36.1.145} that converts the original computationally intractable joint chance constraints into conservative but tractable independent constraints.}
Hence, \textcolor{black}{similar to previous works,} we employ Boole's inequality \cite{doi:10.1287/opre.36.1.145} to get the conservative approximation of \eq{joint_CC} as follows:
\begin{equation}
\begin{split}
\text{Pr}\left(\bigwedge_{k=0}^{N} \left(\bigwedge_{i=1}^{n_c} (\lambda_{k+1, i}, y_{k+1, i}) \in \mathcal{S} \right)\right) \geq 1- \Delta_1, \\
\text{Pr}\left(\bigwedge_{k=0}^{N} \left(\bigwedge_{l=1}^{L}a_l^\top x_k  \leq b_l\right)\right) \geq 1-\Delta_2, \Delta_1 = \Delta_2 = \frac{\Delta}{2} 
\label{joint_CC1}
\end{split}
\end{equation}
Using Boole's inequality again, we can further obtain the conservative chance constraints given by:
\begin{subequations}
\begin{flalign}
\text{Pr}\left( (\lambda_{k+1, i}, y_{k+1, i}) \in \mathcal{S} \right) \geq 1- \frac{\Delta_1}{Nn_c},  \label{8a}\\
\text{Pr}\left(a_l^\top x_k  \leq b_l\right) \geq 1-\frac{\Delta_2}{NL},
\Delta_1 = \Delta_2 = \frac{\Delta}{2} \label{8b}
\end{flalign}
\label{joint_CC2}
\end{subequations}
\textcolor{black}{We discuss how to handle \eq{8a} in Sec~\ref{CCC_explanation}. We formulate \eq{8b} as its \textcolor{black}{equivalent} deterministic form:} 
\begin{subequations}
\begin{flalign}
   \text{Pr}\left(a_l^\top x_k  \leq b_l\right) \geq 1-\frac{\Delta_2}{NL} \Longleftrightarrow \\
a_l^{\top} \bar{x}_k \leq b_l-\sqrt{a_l^{\top} \Sigma_{x_k} a_l} {\Phi}^{-1}(1-\frac{\Delta_2}{NL}) \label{deterministic_CC}
\end{flalign}
\label{chance_analytic}
\end{subequations}
where $\bar{x}_k, \Sigma_{x_k}$ are the mean and covariance matrix of $x_k$, respectively. $\Phi^{-1}$ is an inverse of the cdf of the standard normal distribution. 


\subsection{Chance Complementarity Constraints (CCC) for SDLCS}\label{CCC_explanation}
We make the assumptions as specified in Sec~\ref{assm}.
While a more general formulation could allow the complementarity variable $\lambda_{k+1}$ to be stochastic, we do not consider it here. However, we believe that allowing $C$ and $F$ to be stochastic can achieve a similar effect in SDLCS. Furthermore, in cases where the distribution of $\lambda_{k+1}$ is known, our proposed formulation can be easily extended to incorporate stochasticity in $\lambda_{k+1}$. However, for brevity, we skip these details. 
The Gaussian assumption on uncertainty is made primarily to allow equivalent reformulation of the chance constraints to deterministic inequalities. 

While \cite{DBLP:journals/corr/abs-2105-09973} proposed a promising CCC, their formulation possesses empty solutions when $\Delta \leq 0.5$ (see~\cite{DBLP:journals/corr/abs-2105-09973}).
%
This can result in a very fragile trajectory since the total violation probability over $N$ steps would be always more than 1 if $N\geq 1$ (using Boole's inequality). 
This is because they use a Non-Linear Programming (NLP) formulation which needs to impose all CCC constraints simultaneously which compete with each other.  

  \begin{figure}[t]
    \centering
    \includegraphics[width=0.33\textwidth]{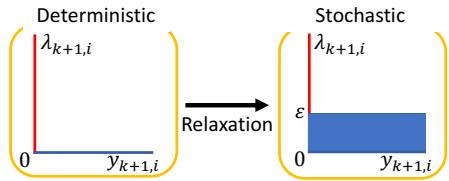}
     \caption{Deterministic and stochastic complementarity constraints. We have the complementarity constraints $0 \leq \lambda_{k+1, i} \perp y_{k+1, i} \geq 0$ where $y_{k+1, i}$ has uncertainty and accepts the violation of $\epsilon$.}
     \label{concept_figure}
\end{figure} 

In our formulation, we decompose stochastic complementarity constraints into two modes  (see \fig{concept_figure}) as follows:
\begin{subequations}
\begin{flalign}
    \text{Pr}\left((\lambda_{k+1,i}, y_{k+1,i}) \in \mathcal{S} \right) \geq 1- \theta  \\
    \Longleftrightarrow \text{Pr} \left( \begin{aligned} 
    \left(\lambda_{k+1, i} \geq 0, y_{k+1, i} = 0  \right) \\
    \bigvee \left(\lambda_{k+1, i} = 0, y_{k+1, i} \geq 0   \right)  \end{aligned} \right) \geq 1 - \theta \\
\Longleftrightarrow \left\{ \begin{aligned} \lambda_{k+1, i} \geq 0, \text{Pr}\left( y_{k+1, i} =0  \right) \geq 1 - \theta \\ \text{or } \lambda_{k+1, i} =0, \text{Pr}\left(y_{k+1, i} \geq 0 \right) \geq 1 - \theta \end{aligned} \right.
\label{disjunctiveset}
\end{flalign}
\end{subequations}
where $\theta = \frac{\Delta_1}{Nn_c}$. Note that now $y_{k+1} \sim \mathcal{N}\left(\bar{y}_{k+1}, \Sigma_{y_{k+1}}\right)$.
To realize lower violation probabilities, we first propose the following CCC using MIP as:
\begin{subequations}
\begin{flalign}
 z_{k, i, 0} \geq 0, \Longrightarrow \lambda_{k+1, i} \geq 0, \text{Pr}\left( y_{k+1, i} =0  \right) \geq 1 - \theta,\label{ccc_ours1}\\
z_{k, i, 1} \geq 0, \Longrightarrow \lambda_{k+1, i} =0, \text{Pr}\left(y_{k+1, i} \geq 0 \right) \geq 1 - \theta \label{ccc_ours2}
\end{flalign}
\label{ccc_ours_0}
\end{subequations}
where $z_{k, i, 0}, z_{k, i, 1}$ represent the integer variables to represent the two modes which satisfies $z_{k,i, 0} + z_{k,i, 1} = 1$ for $i$-th complementarity constraint at instant $k$.  

However, $\text{Pr}\left( y_{k+1, i} =0  \right)$ is zero (as probability measure for singleton sets is zero) so that we cannot directly use \eq{ccc_ours_0}. To alleviate this issue while avoiding negative values for $\lambda$, we propose the following CCC using a relaxation for complementarity constraints (see \fig{concept_figure}): 
\begin{subequations}
\begin{flalign}
z_{k, i, 0} \geq 0, \Longrightarrow \lambda_{k+1, i} \geq 0, \text{Pr}\left(0 \leq y_{k+1, i} \leq \epsilon \right) \geq 1 - \theta,\label{ccc_ours11}\\
 z_{k,i, 1} \geq 0, \Longrightarrow \lambda_{k+1,i} =0, \text{Pr}\left(y_{k+1, i} \geq \epsilon \right) \geq 1 - \theta \label{ccc_ours22}
\end{flalign}
\label{ccc_ours}
\end{subequations}
where  $\epsilon > 0$ is the acceptable violation in the complementarity constraints. 

We have two-sided linear chance constraints in \eq{ccc_ours11}. We decompose \eq{ccc_ours11} as two one-sided chance constraints so that we can use the same reformulation in \eq{chance_analytic}. Note that each one-sided chance constraints, obtained from the two-sided chance constraint, are formulated with a maximum violation probability of $\frac{\theta}{2}$.

Since we have integer constraints, MIP can impose individual constraints for each mode. Thus, we do not need to impose all constraints simultaneously like the NLP formulation in~\cite{DBLP:journals/corr/abs-2105-09973}. This provides a lower bound for $\theta$ as function of $\epsilon$, $\bar{y}_{k+1, i}$, and $\Sigma_{y_{k+1}, ii}$, which is presented as a lemma. 


\begin{lem}
Suppose the CCC are formulated as \eq{ccc_ours} and $\epsilon$, $\bar{y}_{k+1, i}$, and $\Sigma_{y_{k+1}, ii}$ are specified. Then (i) \eq{ccc_ours11} is feasible for all $\theta > 2(1-\Phi(\frac{\epsilon}{2\Sigma_{y_{k+1}, ii}}))$ and (ii) \eq{ccc_ours22} is feasible for all $\theta > 1 - \Phi \left((\bar{y}_{k+1, i} - \epsilon)/\Sigma_{y_{k+1} ii}\right)$.
\label{lemma_1}
\end{lem}
\begin{proof}
Consider case (i): 
From \eq{deterministic_CC} and \eq{ccc_ours11}, the two-side chance constraints in \eq{ccc_ours11} are converted to their deterministic forms:$ \Sigma_{y_{k+1}, ii}  \Phi^{-1} \left(1-\theta/2\right) \leq \bar{y}_{k+1, i} \leq \epsilon -\Sigma_{y_{k+1}, ii}  \Phi^{-1} \left(1-\theta/2\right)$. To have a nonempty solution, $\epsilon - 2 \Sigma_{y_{k+1}, ii}  \Phi^{-1} \left(1-\theta/2\right) > 0$. 
Simplifying this equation, we obtain the bound specified in (i).
Consider case (ii): 
From \eq{deterministic_CC} and \eq{ccc_ours22}, the one-side chance constraints  in \eq{ccc_ours22} are converted as: $\bar{y}_{k+1, i} \geq \epsilon + \Sigma_{y_{k+1}, ii}\Phi^{-1} \left(1-\theta\right)$.
Simplifying this equation, we obtain the bound specified in (ii).
\end{proof}
%
\textit{Remark 1}:
From Lemma~\ref{lemma_1}, it is easy to show that $\theta < \frac{1}{2}$ if $\frac{\epsilon}{2\Sigma_{y_{k+1}, ii}} > \Phi^{-1}(\frac{3}{4})$ for case (i), and if $(\bar{y}_{k+1, i} - \epsilon)/\Sigma_{y_{k+1} ii} > \Phi^{-1}(\frac{1}{2})$ for case (ii). In contrast, the formulation in \cite{DBLP:journals/corr/abs-2105-09973} cannot enforce the chance constraints for any $\theta < 0.5$.


We use the following equations for uncertainty propagation in the SDLCS:
\begin{subequations}
\begin{flalign}
\bar{x}_{k+1}= A\bar{x}_{k} + Bu_k + \overline{C\lambda}_{k+1} + g_k + \bar{w}_k,\\
\Sigma_{x_{k+1}} = A \Sigma_{x_{k}} A^\top + \Sigma_{C\lambda_{k+1}} + W
\end{flalign}
\end{subequations}
where $W$ represents the noise covariance matrix and $\overline{C\lambda}_{k+1}$  represents a mean  of ${C\lambda}_{k+1}$. $\Sigma_{C\lambda_{k+1}} = \mathbb{E}\left[ \left(C\lambda_{k+1} - \overline{C\lambda}_{k+1}\right)\left( {C\lambda_{k+1} - \overline{C\lambda}_{k+1}}\right)^{\top}\right]$, which is a diagonal matrix because of the independence of random variables. Note that $\lambda_{k+1}$ is a decision variable. Consequently, we introduce another simplification by considering the worst-case uncertainty for $\lambda_{k+1}$ during uncertainty propagation. This conservative simplification offers computational advantages during the resulting optimization.



\subsection{Mixed-Integer Quadratic Programming with Chance Constraints (MIQPCC)}\label{MIQP_CCC}
In this section, we present our MIQPCC formulation to solve \eq{equation_slcp}. To impose our proposed CCC, one can solve either MIP or NLP. Our MIP-based method solves disjunctive inequalities while NLP needs to impose all CCC simultaneously, which yields an empty solution for $\Delta \leq 0.5$.


Our proposed MIQPCC is formulated as follows:
\begin{subequations}
\begin{flalign}
\min _{x, u, \lambda, z} \sum_{k=0}^{N-1} \bar{x}_{k}^{\top} Q \bar{x}_{k}+u_{k}^{\top} R u_{k}\\
\text{s. t. }\bar{x}_{k+1}= A\bar{x}_{k} + Bu_k + \overline{C\lambda}_{k+1} + g_k + \bar{w}_k,\\
\Sigma_{x_{k+1}} = A \Sigma_{x_{k}} A^\top + \Sigma_{w, C\lambda_{k+1}} + W,\label{eq000}\\
x_{0} \sim \mathcal{N}\left(x_{s}, \Sigma_{s}\right), u_{k} \in \mathcal{U},  \lambda_k \leq \lambda_{u},\label{eq001}\\
a_l^{\top} \bar{x}_k \leq b_l-\alpha\kappa, \label{eq002}\\
z_{k,i,0} + z_{k,i, 1} =1,\label{eq004}\\
0 \leq \lambda_{k+1, i} \leq M z_{k,i,0}, \label{eq005}\\
\zeta\psi z_{i,k,0} + (\epsilon+\eta\psi) z_{k,i,1} \leq  \bar{y}_{k+1, i} \\
\bar{y}_{k+1, i} \leq (\epsilon-\zeta\psi)z_{k,i,0} + Mz_{k,i,1} ,\label{eq006}
\end{flalign}
\end{subequations}
where $Q=Q^{\top} \geq 0,R=R^{\top} > 0$, $\alpha ={\Phi}^{-1}(1-\frac{\Delta}{2NL}),  \zeta = {\Phi}^{-1}(1-\frac{\Delta}{4Nn_c}), \eta={\Phi}^{-1}(1-\frac{\Delta}{2Nn_c}), \kappa = \sqrt{a_l^{\top} \Sigma_{x_k} a_l},  \psi = \sqrt{\Sigma_{y_{k+1, ii}}}$. $z_{k,i, 0}, z_{k,i, 1}$ are the binary decision variables for the $i$-th complementarity constraint at $k$ to represent mode 1, 2, respectively.  Using these binary variables, we employ big-M formulation to deal with disjunctive inequalities in our CCC. The parameter $M$ is a valid upper bound for $\lambda_k, y_k$. 
\section{Numerical Simulations}\label{sec:result}
We validate our proposed methods for three benchmark DLCS: a cartpole with softwalls, a sliding box with friction, and dual manipulators as illustrated in \fig{fig:example_systems}, inspired by \cite{DBLP:journals/corr/abs-2008-02104}. Through the experiments, we try to answer the following questions: 

%
%
\begin{enumerate}
    \item Can our proposed optimization generate robust open-loop trajectories?
    \item Can our proposed formulation satisfy the probabilistic constraints imposed during optimization?
    \item How does the proposed method compare against the previous methods for robust optimization in SDLCS?
\end{enumerate}

\subsection{Implementation Details}
We implemented our method in Python using Gurobi \cite{gurobi}
\textcolor{black}{to solve the proposed MIQP. We implemented the MPCC with PYROBOCOP \cite{DBLP:journals/corr/abs-2106-03220} to solve the ERM-based method in \cite{drnach2021robust} and the CCC method in \cite{DBLP:journals/corr/abs-2105-09973}.}
  The examples are implemented on a computer with Intel i7-8565U processor.

To verify the robustness of open-loop trajectories obtained from our proposed optimization, we use Monte Carlo simulations.
\textcolor{black}{We propagate the dynamics by finding the roots of the complementarity system with sampled parameters given the control sequence obtained from optimization. We run each case for 1000 trials with different sampled parameters to estimate the probability of failure.}
 Note that unlike the continuous-domain dynamics, we cannot rollout the dynamics for SDLCS with the given control sequences since we do not have the access to $\lambda_{k+1}$. We add the noise sampled from the distribution which was used during optimization.


\textcolor{black}{For simplicity, we show the continuous-time dynamics. We then discretize continuous-time dynamics into discrete-time dynamics using the explicit Euler method with sample time $dt=0.033$. For notation simplicity, we denote $x_0, \Sigma_0$ as mean and covariance matrix at $k=0$ for states of systems, respectively.}


\subsection{Example Details}

\subsubsection{Cartpole with Softwalls}

The continuous-time dynamics with complementarity constraints for the cartpole with softwalls is as follows:
\begin{subequations}
\begin{flalign}
&\dot{x}_{1}=x_{3},
\dot{x}_{2}=x_{4}, 
\dot{x}_{3}=g \frac{m_{p}}{m_{c}} x_{2}+\frac{1}{m_{c}} u_{1}, \\
&\dot{x}_{4}=\frac{g\left(m_{c}+m_{p}\right)}{l m_{c}} x_{2}+\frac{1}{l m_{c}} u_{1}+\frac{\lambda_{1}}{l m_{p}} -\frac{\lambda_{2}}{l m_{p}}, \\
&0 \leq \lambda_{1} \perp l x_{2}-x_{1}+\frac{1}{k_{1}} \lambda_{1}+d \geq 0, \\
&0 \leq \lambda_{2} \perp x_{1}-l x_{2}+\frac{1}{k_{2}} \lambda_{2}+d \geq 0
\end{flalign}
\label{cartpole-equation}
\end{subequations}
%
where $x_{1}$ is the cart position, $x_{2}$ is the pole angle, the $x_{3}$ and $x_{4}$ are their derivatives. $u_{1}$ is the control and $\lambda_{1}, \lambda_{2}$ are the reaction forces at from the wall 1, 2, respectively. We consider the additive noise $w$, the zero-mean i.i.d. Gaussian noise which standard deviation is $2\times 10^{-4}$, to $x_{1, k}, x_{2, k}$. $k_{1}=10, k_{2}=10$ are the stiffness of walls $1$ and  $2$, respectively.  In this example, we assume that the uncertainty also arises from the $\frac{1}{k_1}, \frac{1}{k_2}$ which standard deviations are $10^{-5}$. $g=9.81$ is the gravitational acceleration, $m_p=0.1, m_c=1.0$ are the mass of the  pole, cart, respectively. $l=0.5$ is the length of the pole and $d=0.15$ is the distance from the origin of the coordinate to the walls.  

The optimization setup is as follows. $N=20, M=100, Q=\text{diag}(0,0,0,0), R=0.01$, $\epsilon=0.002,$ $x_0 = [-0.15, 0, 0, 0]^\top, \Sigma_{0}=\text{diag}(0,0,0,0)$.
We also have the following chance constraints:  $\text{Pr} ( x_{1,k} \leq 0.05) \geq 1 - \frac{\Delta}{4N}, \text{Pr} ( x_{2,k} \leq 0.15) \geq 1 - \frac{\Delta}{4N},  \forall k = 0, \ldots, N-1$,  $\text{Pr} (-0.02 \leq x_{1, N} \leq 0.02) \geq 1 - \frac{\Delta}{4N}, \text{Pr} (-0.04 \leq x_{2, N} \leq 0.04) \geq 1 - \frac{\Delta}{4N}$. 


\begin{figure}[t]
    \centering
    \includegraphics[width=0.4\textwidth]{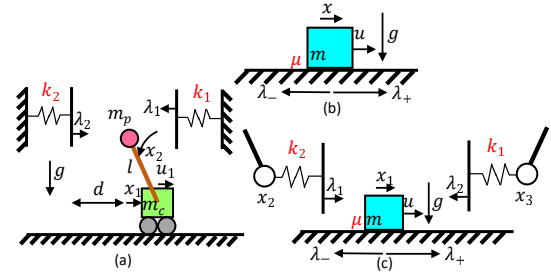}
    \caption{Problems described in Sec~\ref{sec:result}. (a) cartpole with softwalls, (b)sliding box with friction, and (c)dual manipulators.}
    \label{fig:example_systems}
\end{figure}

\subsubsection{Sliding Box with Friction}
The continuous-time \textcolor{black}{quasi-static}
dynamics with complementarity constraints for sliding box with Coulomb friction is as follows:
\begin{subequations}
\begin{flalign}
&\dot{x}_1=x_{2},
\alpha \dot{{x}}_1=u+\lambda_{+}-\lambda_{-}, \\
&0 \leq \gamma \perp \mu m g-\lambda_{+}-\lambda_{-} \geq 0, \\
&0 \leq \lambda_{+} \perp \gamma+u+\lambda_{+}-\lambda_{-} \geq 0, \\
&0 \leq \lambda_{-} \perp \gamma-u-\lambda_{+}+\lambda_{-} \geq 0
\end{flalign}
\label{pushing_dynamics}
\end{subequations}
%
$x_{1}$ is the box position and $x_{2}$ is the box velocity. $u$ is the control and $\lambda_{+}, \lambda_{-}$ are the positive and negative component of the friction force, respectively. $\gamma$ is the slack variable. $\alpha=4$ is the damping constant, $m=1$ is the mass of the box, and $\mu=0.1$ is the coefficient of friction.  We consider the additive  i.i.d. Gaussian noise $w$ as  ${x_{1, k+1}}=x_{1,k} + x_{2, k} dt + w$.   The standard deviations of $w$ is $4\times 10^{-4}$. $g=9.81$ is the gravitational acceleration. We assume that the uncertainty also arises from the $\mu$ which standard deviations are $10^{-5}$.


The optimization setup is as follows. $N=20, M=100, Q=\text{diag}(0,0,0,0), R=0.01$, $\epsilon = 0.01, x_0 = [1, -1]^\top, \Sigma_{0}=\text{diag}(0,0)$. 
We also have the following chance constraints:  $\text{Pr} ( x_{1,k} \geq 0.885) \geq 1 - \frac{\Delta}{2N}, \forall k = 0, \ldots, N-1$, $\text{Pr} (0.89 \leq x_{1, N} \leq 0.91) \geq 1 - \frac{\Delta}{2N}, \text{Pr} (-0.1 \leq x_{2, N} \leq 0.1) \geq 1 - \frac{\Delta}{2N}$.


\subsubsection{Dual Manipulators}
We consider the example where the box is manipulated by two manipulators with Coulomb friction and the contact forces from the manipulators. 
The continuous-time \textcolor{black}{quasi-static} dynamics is as follows:
\begin{equation}
\begin{aligned}
&\dot{x}_1=x_{2}, \alpha \dot{x}_{1}=\lambda_{1}-\lambda_{2}+\lambda_{+}-\lambda_{-},\\
&\dot{x}_3= x_{4},
\dot{x}_4 = u_{1},
\dot{x}_5= x_{6},
\dot{x}_6 = u_{2},\\
&0 \leq \lambda_{1} \perp x_{1}-x_{3}+\frac{1}{k} \lambda_{1} \geq 0, \\
&0 \leq \lambda_{2} \perp x_{5}-x_{1}+\frac{1}{k} \lambda_{2} \geq 0, \\
&0 \leq \gamma \perp \mu m g-\lambda_{+}-\lambda_{-} \geq 0, \\
&0 \leq \lambda_{+} \perp \gamma+\lambda_{1}-\lambda_{2}+\lambda_{+}-\lambda_{-} \geq 0, \\
&0 \leq \lambda_{-} \perp \gamma-\lambda_{1}+\lambda_{2}-\lambda_{+}+\lambda_{-} \geq 0
\end{aligned}
\label{dual_manipulation_eq}
\end{equation}
$x_{1}, x_{3}, x_{5}$ are the positions of the box, the left arm, the right arm, respectively and $x_{2}, x_{4}, x_{6}$ are their derivatives. $u_{1}, u_{2}$ represent the controls of the left and the right arm, respectively. $\lambda_{+}, \lambda_{-}$ are the positive and negative component of the friction force, respectively. $\gamma$ is the slack variable. $\lambda_{1}, \lambda_{2}$ are the contact forces from the left arm and the right arm, respectively. We set $g=9.81$, $m=1$, $k=100$, $\mu=0.1$. We discretize the dynamics \eq{dual_manipulation_eq} with $dt=0.033$ and add the zero-mean i.i.d. Gaussian noise $w$ which standard deviation is 0.0002 such as ${x_{1, k+1}}=x_{1,k} + x_{2, k} dt + w$. The standard derivation of $\mu$ and $\frac{1}{k}$  are 0.0001.


The optimization setup in this example is as follows. $N=20, M=50, Q=\text{diag}(0,0,0,0,0,0), R=\text{diag}(1,1), \epsilon=0.0042$, $x_0 =[0.1, -1.1, 0,0,0.1,0]^\top, \Sigma_{0}=\text{diag}(0,0,0,0,0,0)$. We have the following chance constraints:  $\text{Pr} ( x_{1,k} \geq -0.17) \geq 1 - \frac{\Delta}{2N}, \forall k = 0, \ldots,  N-1$, $\text{Pr} (-0.01 \leq x_{1, N} \leq 0.01) \geq 1 - \frac{\Delta}{2N}$. 

\subsection{Robustness of Open-Loop Trajectories}
The optimized control and state trajectories for the three systems using our proposed method are shown in  \fig{fig:cartpole_vioprob_control}-\fig{fig:2dmanipulation_vioprob_x}. Overall, these figures show that the planner generates state trajectories that are farther away from the bound specified in the chance constraints as the violation probability decreases.  For instance, \fig{fig:cartpole_vioprob_control} shows that the trajectories are farther away from $x=0.05$ as $\Delta$ decreases. In addition, the trajectory with $\Delta=0.02$ reaches its maximum value earlier than other trajectories to account for the evolution of the uncertainty. We observe the same behavior for the other example too. In addition, these figures illustrate that the control costs increase as $\Delta$ decreases. This illustrates the trade-off relation between safety and cost.

At this point, we would like to discuss the magnitude of uncertainty we consider in these problems. Compared to some other stochastic optimal control works \cite{5970128, 8651541}, the uncertainty in these problems is relatively smaller. There are several reasons why we need to have a smaller uncertainty. Note that as we have explained in Sec~\ref{sec:prob_formulation}, our approach satisfies joint constraints on multiple constraints together. First, our formulation has chance complementarity constraints in addition to chance constraints on states, which are commonly used. Our formulation has more number of chance constraints, and consequently, the lower uncertainty is required because of the conservative approximation of Boole's inequality to resolve joint chance constraints into individual constraints as explained  in Sec~\ref{ccc_section},  Sec~\ref{CCC_explanation}. Second, we need to have a small $\epsilon$ to avoid large violation of complementarity constraints, which requires small uncertainty. Finally, we would like to emphasize that allowing larger uncertainties requires either better resolution of joint chance constraints or covariance steering approaches \cite{8651541}, which is out of scope for the current study.


\begin{figure}[t]
    \centering
    \includegraphics[width=0.42\textwidth]{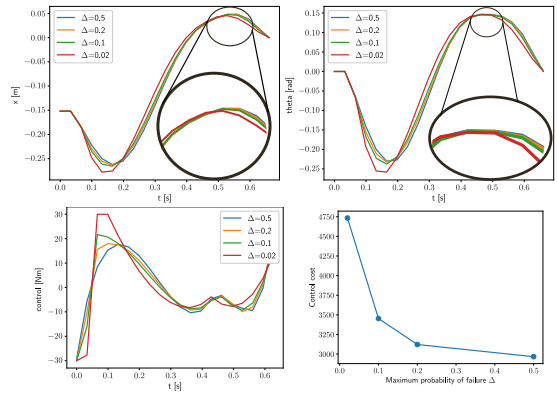}
    \caption{State and control trajectories with different $\Delta$ for the cartpole example. First, the cart moves in the negative direction to utilize the contact force $\lambda_2$ because the control input is bounded. Once the cart obtains enough $\lambda_2$, the cart is accelerated to the positive direction. We can observe the effect of our proposed chance constraints in particular around $t \in [0, 0.1]$ and $t \in [0.4, 0.5]$. When $t \in [0, 0.1]$, the mode changes from the "contact on the wall 2" to the "no contact" and the cart tries to be far from wall 2 to satisfy the CCC. When $t \in [0.4, 0.5]$, the trajectories are farther away from $x_1=0.05$ and $x_2=0.15$ as $\Delta$ decreases.}
    \label{fig:cartpole_vioprob_control}
\end{figure}

\begin{figure}[t]
    \centering
    \includegraphics[width=0.43\textwidth]{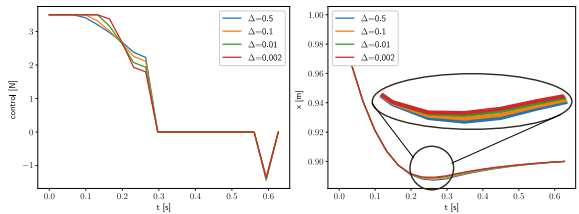}
    \caption{State and control trajectories with different $\Delta$ for a sliding box with friction. First, the box is accelerated in the positive direction. Then, the control decreases with time to regulate the box around the origin by employing the friction forces. We can observe the effect of our proposed chance constraints in particular around $t\in [0.2, 0.3]$ where the trajectories are farther away from $x_1=0.88$ as $\Delta$ decreases.}
    \label{fig:pushingbox_vioprob_control}
\end{figure}

\begin{figure}
    \centering
    \includegraphics[width=0.43\textwidth]{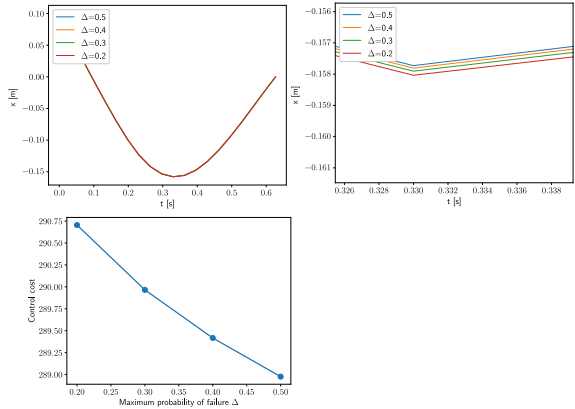}
    \caption{Time history of $x_1$ with different maximum violation probabilities $\Delta$ for dual manipulation. First, the box is pushed by the right arm in the negative direction. Next, the left arm regulates the box to the origin. In particular, around $t\in 0.2-0.3$ s, the trajectories are farther away from $x_1=-0.17$ as $\Delta$ decreases.}
    \label{fig:2dmanipulation_vioprob_x}
\end{figure}

\subsection{Monte Carlo Simulation Results}

\tab{table_cartpole}-\tab{table_dualmanipulation} show our Monte Carlo simulation results given the control sequences with different $\Delta$ from the optimization compared to the ERM method in \cite{drnach2021robust} and the CCC method in \cite{DBLP:journals/corr/abs-2105-09973}. We run the ERM method in \cite{drnach2021robust} with different weighting $\beta$ and the CCC  \cite{DBLP:journals/corr/abs-2105-09973} with violation probability $\Delta_z=0.5$. $\beta$ was chosen so that the magnitude of the ERM cost is a similar order of other costs. For a fair comparison, we regard that the constraints are violated if the chance constraints are not satisfied in our method. We regard that the constraints are violated in the ERM in \cite{drnach2021robust} and the CCC methods in \cite{DBLP:journals/corr/abs-2105-09973} if the terminal chance constraints used in our proposed method are not satisfied.

\begin{table}[t]
    \caption{{Comparison of our specified $\Delta$ in optimization, specified $\beta$ in ERM in \cite{drnach2021robust}, and the CCC in \cite{DBLP:journals/corr/abs-2105-09973} with $\Delta_z=0.5$,  and obtained $\Delta$ from the simulation of "cartpole with softwalls" over 1000 samples.}}
    \centering
    \begin{tabular}{c|c|c|c|c}
      & $\Delta=0.5$& $\Delta=0.2$ & $\Delta=0.1$ & $\Delta=0.02$\\
         \hline
         Obtained $\Delta$ & 0.190 & 0.147 & 0.085 & 0.020
    \end{tabular}
        \begin{tabular}{c|c|c|c|c}
      & $\beta=10^3$& $\beta=10^4$ & $\beta=10^5$ & $\Delta_z=0.5$\\
         \hline
         Obtained $\Delta$ & 0.75 & 1.0 & 1.0 & 0.91
    \end{tabular}
    \label{table_cartpole}
\end{table}

\begin{table}[t]
    \caption{{Comparison of our specified $\Delta$ in optimization, specified $\beta$ in ERM in \cite{drnach2021robust}, and the CCC in \cite{DBLP:journals/corr/abs-2105-09973} with $\Delta_z=0.5$,  and obtained $\Delta$ from the simulation of "a sliding box with friction" over 1000 samples.}}
    \centering
    \begin{tabular}{c|c|c|c|c}
      & $\Delta=0.5$& $\Delta=0.1$ & $\Delta=0.01$ & $\Delta=0.002$\\
         \hline
         Obtained $\Delta$ & 0.080 & 0.051 & 0.027 & 0.010
    \end{tabular}
        \begin{tabular}{c|c|c|c|c}
      & $\beta=10^3$& $\beta=10^4$ & $\beta=10^5$ & $\Delta_z=0.5$\\
         \hline
         Obtained $\Delta$ & 1.0 & 1.0 & 1.0 & 0.91
    \end{tabular}
    \label{table_pushingbox}
\end{table}

\begin{table}[t]
    \caption{{Comparison of our specified $\Delta$ in optimization, specified $\beta$ in ERM, and the CCC with $\Delta_z=0.5$,  and obtained $\Delta$ from the simulation of "dual manipulation" over 1000 samples.}}
    \centering
    \begin{tabular}{c|c|c|c|c}
      & $\Delta=0.5$& $\Delta=0.4$ & $\Delta=0.3$ & $\Delta=0.2$\\
         \hline
         Obtained $\Delta$ & 0.419 & 0.317 & 0.257 & 0.217
    \end{tabular}
        \begin{tabular}{c|c|c|c|c}
      & $\beta=10^3$& $\beta=10^4$ & $\beta=10^5$ & $\Delta_z=0.5$\\
         \hline
         Obtained $\Delta$ & 1.0 & 1.0 & 1.0 & 1.0
    \end{tabular}
    \label{table_dualmanipulation}
\end{table}

\tab{table_cartpole} shows that the empirically obtained violation probabilities are lower than the specified violation probabilities used in our proposed optimization. In contrast, the control sequences based on the ERM method in \cite{drnach2021robust} show $100 \%$ violation probabilities with $\beta=10^4, 10^5$, which are much worse than the obtained violation probabilities using our proposed method. With $\beta=10^3$, we got a relatively good violation probability. The CCC in \cite{DBLP:journals/corr/abs-2105-09973} could also show the relatively good violation probability compared to the ERM-based method with $\beta=10^4, 10^5$ but shows the worse violation probability compared to our method with $\Delta=0.5$ and the ERM with $\beta=10^3$. 
Thus, we confirm that our proposed approach satisfies chance constraints in the simulator in this example. 

\tab{table_pushingbox} shows that empirically obtained violation probabilities are lower than the specified violation probabilities used in our proposed optimization like the cartpole example, except for the cases with $\Delta=0.01, 0.002$. There are several factors that contribute to the violation of the chance constraints. Unlike the cartpole example, $F$ is not a P matrix so we can get the multiple solutions in $\lambda$, which can lead to non-Gaussian distributions. Also, even though $\epsilon$ is small, it is not zero so the actual trajectory in the simulator cannot be exactly the same as the trajectory from the optimization even with no noise. While we can ignore these effects with relatively large $\Delta$, we cannot ignore these effects anymore with the small $\Delta$. Although the planner could not satisfy the chance constraints for all $\Delta$ in this example, our method achieves much lower violation probabilities compared to the ERM in \cite{drnach2021robust} and the CCC in \cite{DBLP:journals/corr/abs-2105-09973}.
\tab{table_dualmanipulation} shows that we have the same discussion for the dual manipulators example as for the pushing a box example.

\fig{sim_cartpole} and \fig{sim_pushing_box} show that our proposed planner could successfully drive the system to the goal state. We also observe that with decreasing $\Delta$, the system trajectories move further away from state set boundaries to satisfy tighter chance constraints.
%
For \fig{2dmanipulation_vioprob_theta_mc_delta02}, while the majority of the sampled trajectories converge to the specified terminal constraints, some of them clearly converged to other states. This result also shows that the true distribution of the uncertainty for the dynamics systems with LCS is not Gaussian.

   \begin{figure}[t]
    \centering
    \includegraphics[width=0.45\textwidth]{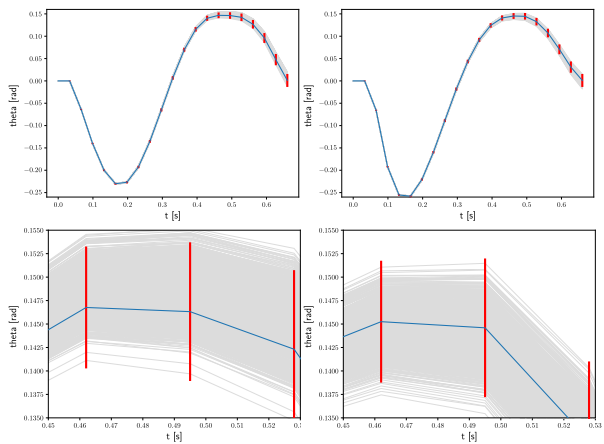}
     \caption{Simulated trajectories of $x_2$ of the cartpole example over 1000 samples with $\Delta = 0.5$ for the left column and with $\Delta=0.02$ for the right column. The bottom row enlarges the top row figures around the area where the chance constraints effect is observed. The red line shows the 99.9 $\%$ confidence interval.}
     \label{sim_cartpole}
\end{figure}

   \begin{figure}[t]
    \centering
    \includegraphics[width=0.4\textwidth]{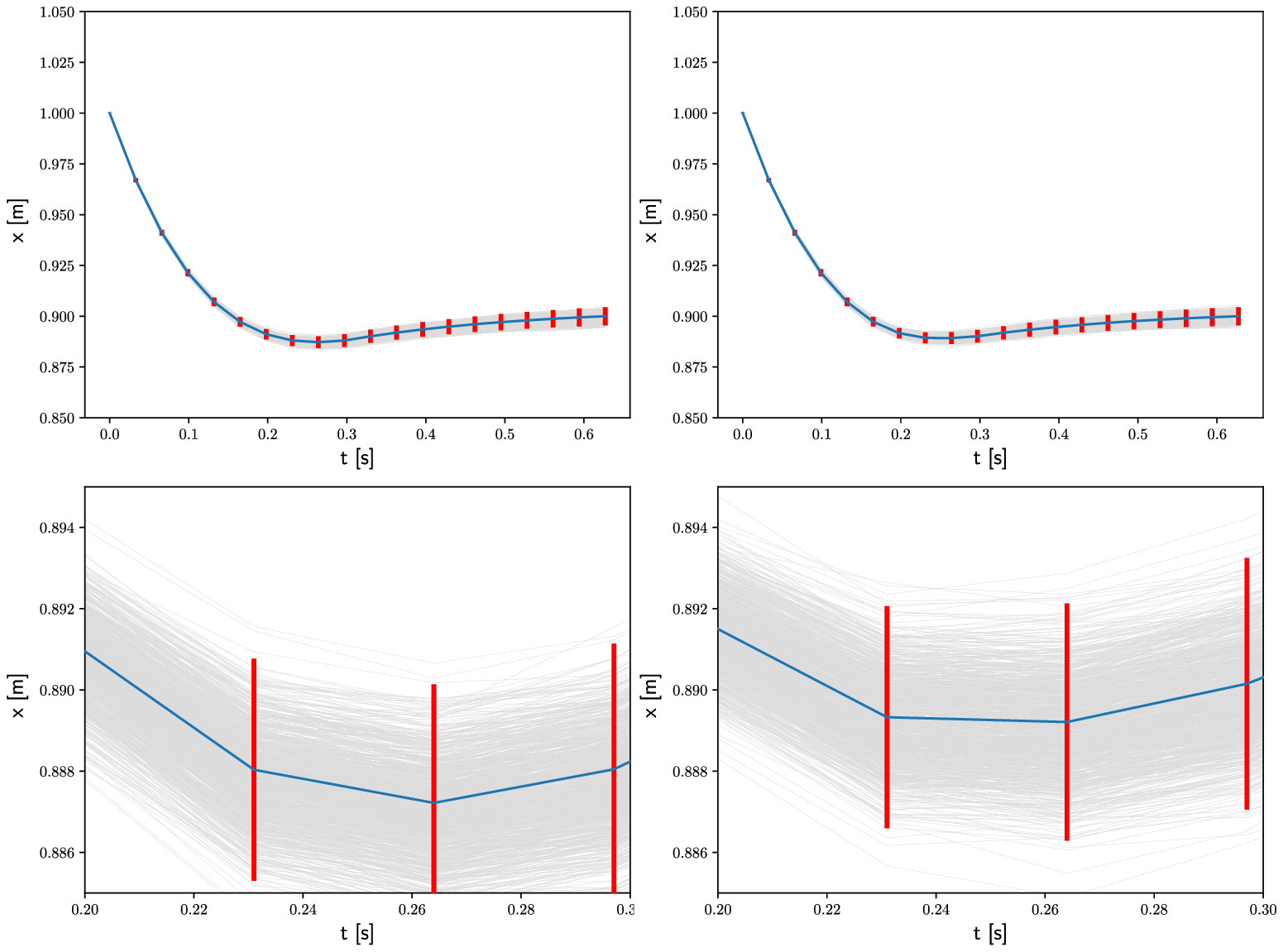}
     \caption{Simulated trajectories of $x$ of sliding box with friction example over 1000 samples with $\Delta = 0.5$ for the left column and with $\Delta=0.002$ for the right column. The bottom row enlarges the top row figures around the area where the chance constraints effect is observed. The red line shows the 99 $\%$ confidence interval.}
     \label{sim_pushing_box}
\end{figure}

\begin{figure}
    \centering
    \includegraphics[width=0.4\textwidth]{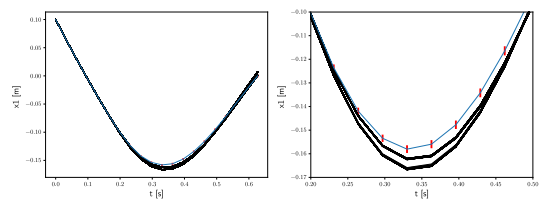}
    \caption{Simulated trajectories of $x_1$ over 1000 samples with $\Delta = 0.2$ for dual manipulation. The right figure shows the enlarged figure of the left figure. The red line shows the 99 $\%$ confidence interval.}
    \label{2dmanipulation_vioprob_theta_mc_delta02}
\end{figure}

\section{Discussion and Conclusion}\label{sec:discussion}
The hybrid dynamics of frictional interaction as well as uncertainty associated with frictional parameters make the efficient design of model-based controllers for manipulation challenging. 
In this paper, we presented a robust TO technique for contact-rich systems. We presented a formulation for chance constrained optimization for SDLCS which is solved using MIQPCC. We compared our proposed approach against other recent techniques for robust optimization for stochastic complementarity systems. We showed that our formulation leads to more robust trajectories compared to these techniques. 

In the future, we would like to relax certain assumptions in this work. We would like to propose solutions for general non-linear stochastic complementarity systems in the presence of non-Gaussian noise. In the current work, using joint chance constraints on all the variables results in conservative solutions. To consider these problems, the study of nonlinear uncertainty propagation in SNCS is required. Also, we need to solve mixed-integer non-linear programming.  We would also like investigate how we can relax the conservative solutions obtained by our proposed approach using better measures for risk. We would also like to incorporate real-time sensor input~\cite{dong2021icra} to develop algorithms for stochastic model predictive control of complex manipulation problems~\cite{yuki2022pivot}. Another interesting line of work would to be to include a Reinforcement learning algorithm to get model updates~\cite{9387127} during learning.

\bibliographystyle{IEEEtran}
\bibliography{references}
\end{document}